\documentclass[12pt]{article}

\usepackage{amsthm}
\usepackage{amsmath}
\usepackage{amssymb}
\usepackage{graphicx}
\usepackage{url}
\usepackage{algorithm2e}

\usepackage[capitalise]{cleveref}
\usepackage[mathscr]{euscript}
\usepackage{dsfont}
\usepackage{braket}
\usepackage{calrsfs}
\DeclareMathAlphabet{\pazocal}{OMS}{zplm}{m}{n}

\DeclareMathOperator{\supp}{supp}

\DeclareMathOperator{\cv}{conv}
\newcommand{\C}{\ensuremath{\pazocal{C}}}

\newcommand{\N}{\ensuremath{\mathbb{N}}}

\newcommand{\R}{\ensuremath{\mathbb{R}}}

\newcommand{\brbm}{\ensuremath{\mathsf{B}\textup{-}\mathsf{RBM}}}
\newcommand{\dbn}{\ensuremath{\mathsf{DBN}}}

\newcommand{\cH}{\ensuremath{\mathcal{H}}}
\newcommand{\cZ}{\ensuremath{\pazocal{Z}}}
\newcommand{\cX}{\ensuremath{\pazocal{X}}}
\renewcommand{\t}{\ensuremath{\mathfrak{t}}}

\newcommand{\1}{\ensuremath{\mathds{1}}}

\newcommand{\Expec}[1]{\mathbb{E}\left[#1\right]}

\newcommand{\prob}{\ensuremath\mathbb{P}}

\renewcommand{\geq}{\geqslant}
\renewcommand{\leq}{\leqslant}

\crefname{proposition}{proposition}{propositions}
\Crefname{proposition}{Proposition}{Propositions}

\usepackage{times}

\newtheorem{theorem}{Theorem}[section]
\newtheorem{remark}[theorem]{Remark}
\newtheorem{corollary}[theorem]{Corollary}
\newtheorem{lemma}[theorem]{Lemma}

\newtheorem{proposition}[theorem]{Proposition}
\newtheorem{example}[theorem]{Example}

\begin{document}
	
\title{ \vspace{-2.5cm} Quantitative Universal Approximation Bounds for Deep Belief Networks}	
	
\author{%
Julian Sieber \footnote{j.sieber19@imperial.ac.uk}  \\
Department of Mathematics, Imperial College London, United Kingdom \and
	Johann Gehringer \footnote{j.gehringer18@imperial.ac.uk} \\
	{Department of Mathematics, Imperial College London, United Kingdom}
}

\maketitle

\begin{abstract}%
\noindent
  We show that deep belief networks with binary hidden units can approximate any multivariate probability density under very mild integrability requirements on the parental density of the visible nodes. The approximation is measured in the $L^q$-norm for $q\in[1,\infty]$ ($q=\infty$ corresponding to the supremum norm) and in Kullback-Leibler divergence. Furthermore, we establish sharp quantitative bounds on the approximation error in terms of the number of hidden units.
\end{abstract}


%\begin{keywords}
%  Deep belief network, restricted Boltzmann machine, universal approximation property, expressivity, Kullback-Leibler approximation, probability density
%\end{keywords}

\section{Introduction}

Deep belief networks (DBNs) are a class of generative probabilistic models  obtained by stacking several restricted Boltzmann machines (RBMs, \cite{Smolensky1986}). For a brief introduction to RBMs and DBNs we refer the reader to the survey articles \cite{fischer2012introduction,Fischer2014,Montufar2018,Ghojogh2021}.  Since their introduction, see \cite{Hinton2006-1,Hinton2006-2}, DBNs have been successfully applied to a variety of problems in the domains of natural language processing \cite{Hinton2009,Mingyang2018}, bioinformatics  \cite{Yuhao2013,Liang2015,Renzhi2016,Ping2019}, financial markets \cite{Shen2015} and computer vision \cite{Zaher2015,Kamada2018,Kamada2019,Huang2019}. However, our theoretical understanding of the class of continuous probability distributions, which can be approximated by them, is limited. The ability to approximate a broad class of probability distributions---usually referred to as \emph{universal approximation property}---is still an open problem for DBNs with real-valued visible units. As a measure of proximity between two real-valued probability density functions, one typically considers the $L^q$-distance or the Kullback-Leibler divergence.

\paragraph{Contributions.} In this article we study the approximation properties of deep belief networks for multivariate continuous probability distributions which have a density with respect to the Lebesgue measure. We show that, as $m\to\infty$, the universal approximation property holds for binary-binary DBNs with two hidden layers of sizes $m$ and $m+1$, respectively. Furthermore, we provide an explicit quantitative bound on the approximation error in terms of $m$. More specifically, the main contributions of this article are:

\begin{itemize}
	\item For each $q\in[1,\infty)$ we show that DBNs with two binary hidden layers and parental density $\varphi:\R^d\to\R_+$ can approximate any probability density $f:\R^d\to\R_+$ in the $L^q$-norm, solely under the condition that $f,\varphi \in L^q(\R^d)$. In addition, we prove that the error admits a bound of order $\pazocal{O}\big(m^{1-\frac{1}{\min(q,2)}}\big)$ for each $q\in(1,\infty)$, where $m$ is the number of hidden neurons. 
	
	\item If the target density $f$ is uniformly continuous and the parental density $\varphi$ is bounded, we provide an approximation result in the $L^{\infty}$-norm (also known as supremum or uniform norm).  
	
	\item Finally, we show that continuous target densities supported on a compact subset of $\R^d$ and uniformly bounded away from zero can be approximated by deep belief networks with bound\-ed parental density in Kullback-Leibler divergence. The approximation error in this case is of order $\pazocal{O}\big(m^{-1}\big)$.
\end{itemize}

\paragraph{Related works.} One of the first approximation results for deep belief networks is due to \cite{Hinton2008} and states that any probability distribution on $\{ 0,1\}^d$ can be learnt by a DBN with $3\times 2^d$ hidden layers of size $d+1$ each. This result was improved by \cite{LeRoux2010,Montufar2011} by reducing the number of layers to $\frac{2^{d-1}} {d-\log(d)}$ with $d$ hidden units each. These results, however, are limited to discrete probability distributions. Since most applications involve continuous probability distributions, \cite{Krause2013} considered Gaussian-binary DBNs and analyzed their approximation capabilities in Kullback-Leibler divergence, albeit without a rate. In addition, they only allow for target densities that can be written as an infinite mixture of a set of probability densities satisfying certain conditions, which appear to be hard to check in practice. 

Similar questions have been studied for a variety of neural network architectures: The famous results of \cite{Cybenko1989,Hornik1989} state that deterministic multi-layer feed-forward networks are universal approximators for a large class of Borel measurable functions, provided that they have at least one sufficiently large hidden layer. See also the articles \cite{Leshno1993,Chen1995,Barron1993,Burger2001}.
\cite{LeRoux2008} proved the universal approximation property for RBMs and  discrete target distributions. \cite{Montufar2015-1} established the universal approximation property for discrete restricted Boltzmann machines. \cite{Montufar2015-2} showed the universal approximation property for deep narrow Boltzmann machines. \cite{Montufar2015-3} showed that Markov kernels can be approximated by shallow stochastic feed-forward networks with exponentially many hidden units. \cite{Bengio2011,Pascanu2013} studied the approximation properties of so-called deep architectures. \cite{Merkh2019} investigated the approximation properties of stochastic feed-forward neural networks. 

The recent work \cite{Johnson2019} nicely complements the aforementioned results by obtaining an illustrative negative result: Deep narrow networks with hidden layer width at most equal to the input dimension do not posses the universal approximation property.

\section{Deep Belief Networks}

A restricted Boltzmann machine (RBM) is a an undirected, probabilistic, graphical model with bipartite vertices that are fully connected with the opposite class. To be more precise, we consider a simple, planar graph $\pazocal{G}=(\pazocal{V},\pazocal{E})$ for which the vertex set $\pazocal{V}$ can be partitioned into sets $V$ and $H$ such that the edge set is given by $\pazocal{E}=\big\{\{s,t\}:\,s\in V,\,t\in H\big\}$. We call vertices in $V$ \emph{visible} units; $H$ contains the \emph{hidden} units. To each of the visible units we associate the state space $\Omega_V$ and to the hidden ones we associate $\Omega_H$. We equip $\pazocal{G}$ with a \emph{Gibbs probability measure}
\begin{equation*}
    \pi(v,h)=\frac{e^{- \cH(v,h)}}{\cZ},\qquad v\in(\Omega_V)^V,\,h\in(\Omega_H)^H,
\end{equation*}
where $\cH:(\Omega_V)^V\times(\Omega_H)^H\to\R$ is chosen such that $\cZ=\iint e^{-\cH(v,h)}\,dv\,dh<\infty$. Notice that the integral becomes a sum if $\Omega_V$ (resp. $\Omega_H$) is a discrete set. It is customary to identify the RBM with the probability measure $\pi$.

An important example are \emph{binary-binary} RBMs. These are obtained by choosing $\Omega_V=\Omega_H=\{0,1\}$ and 
\begin{equation}\label{eq:rbm_density}
  \cH=\braket{v,Wh}+\braket{v,b}+\braket{h,c},\quad v\in\{0,1\}^V,\,h\in\{0,1\}^H,
\end{equation}
where $b\in\{0,1\}^V$ and $c\in\{0,1\}^H$ are called \emph{biases}, and $W\in\R^{V\times H}$ is called the \emph{weight matrix}. We shall write for $m,n\in\N$,
\begin{equation}\label{eq:rbm_set}
    \brbm(m,n)=\left\{\pi\textup{ is a binary-binary RBM with $m$ visible and $n$ hidden units}\right\},
\end{equation}
for the set of binary-binary RBMs with fixed layer sizes.

The following discrete approximation result is well known, see also \cite{Montufar2011}:

\begin{proposition}[{\cite{LeRoux2008}, Theorem 2}]\label{prop:approx_leroux}
    Let $m\in\N$ and $\mu$ be a probability distribution on $\{0,1\}^m$. Let 
    \begin{equation*}
        \supp(\mu)=\big\{v\in\{0,1\}^m:\,\mu(v)>0\big\}
    \end{equation*}
    be the support of $\mu$. Set $n=\big|\supp(\mu)\big|+1$. Then, for each $\varepsilon>0$, there is a $\pi\in\brbm(m,n)$ such that
  \begin{equation*}
      \left|\mu(v)-\sum_{h\in\{0,1\}^{n}}\pi(v, h)\right|\leq\varepsilon\qquad\forall\,v\in\{0,1\}^m.
  \end{equation*}
\end{proposition}

A deep belief network (DBN) is constructed by stacking two RBMs. To be more precise, we now consider a tripartite graph with hidden layers $H_1$ and $H_2$ and visible units $V$. We assume that the edge set is now given by $\pazocal{E}=\big\{\{s,t_1\},\{t_1,t_2\}:\,s\in V,\,t_1\in H_1,\,t_2\in H_2\big\}$. The state spaces are now $\Omega_V=\R$ and $\Omega_{H_1}=\Omega_{H_2}=\{0,1\}$. We think of edges in the graph as dependence of the neurons (in the probabilistic sense). The topology of the graph hence shows that the vertices in $V$ and $H_2$ shall be conditionally independent, that is, we require that the probability distribution of a DBN satisfies
\begin{equation}\label{eq:dbn_equation}
    p(v,h_1,h_2)=p(v\,|\,h_1)p(h_1,h_2).
\end{equation}
The joint density of the hidden units $p(h_1,h_2)$ will be chosen as binary-binary RBM. 

Let $\mathscr{D}(\R^d)=\big\{f:\R^d\to\R_+:\,\int_{\R^d}f(x)\,dx=1\big\}$ be the set of probability densities on $\R^d$. For $\varphi\in\mathscr{D}(\R^d)$ and $\sigma>0$ we set
\begin{equation}\label{eq:visible_density}
  \mathscr{V}_\varphi^{\sigma}=\left\{\varphi_{\mu,\sigma}=\sigma^{-d}\varphi\left(\frac{x-\mu}{\sigma}\right):\,\mu\in\R^d\right\}.
\end{equation}
Notice that all elements of $\mathscr{V}_\varphi^{\sigma}$ are themselves probability distributions. We fix a \emph{parental density} $\varphi\in\mathscr{D}(\R^{|V|})$ and choose the conditional density in \eqref{eq:dbn_equation} as $p(\cdot\,|\,h_1)\in\mathscr{V}_\varphi^{\sigma}$ for each $h_1\in H_1$.

\begin{example}\label{ex:visible_density}
    The most popular choice of the parental function $\varphi$ in \eqref{eq:visible_density} is the $d$-dimensional standard Gaussian density
    \begin{equation}\label{eq:gaussian}
        \varphi(x)=\frac{1}{(2\pi)^{d/2}}\exp\left(-\frac{|x|^2}{2}\right),\qquad x\in\R^d.
    \end{equation}
    Another density considered in previous works is the truncated exponential distribution 
    \begin{equation}\label{eq:exponential}
        \varphi(x)=\prod_{i=1}^d\frac{\lambda_i e^{-\lambda_i x_i}}{1-e^{-b_i\lambda_i}}\1_{[0,b_i]}(x_i),\qquad x=(x_1,\dots,x_d)\in\R^d,
    \end{equation}
    where $b_i, \lambda_i >0$ for each $i=1,\dots,d$.
\end{example}

Similar to \eqref{eq:rbm_set}, we collect all DBNs in the set
\begin{align*}
    \dbn_\varphi(d,m,n)=\Big\{&p\textup{ is a DBN with parental density $\varphi$, $d$ visible units, $m$ hidden}\\  
    &\textup{ units on the first level, and $n$ hidden units on the second level}\Big\},
\end{align*}
where $\varphi\in\mathscr{D}(\R^d)$ and $d,m,n\in\N$.

\section{Main Results}

To state the results of this article, we need to introduce three bits of additional notation:

Let $q\in[1,\infty)$. The space of $q$-summable functions is denoted by
\begin{equation*}
    L^q(\R^d)=\left\{f:\R^d\to\R:\,\|f\|_{L^q}=\left(\int_{\R^d}\big|f(x)\big|^q\,dx\right)^{\frac1q}<\infty\right\}.
\end{equation*}
We also set
\begin{equation*}
    L^\infty(\R^d)=\left\{f:\R^d\to\R:\,\|f\|_{L^\infty}=\sup_{x\in\R^d}\big|f(x)\big|<\infty\right\}.
\end{equation*}
It is convenient to declare $\mathscr{D}_q(\R^d)=\mathscr{D}(\R^d)\cap L^q(\R^d)$. Finally, let us abbreviate the constant
\begin{equation}\label{eq:cq}
    \Upsilon_q=\max\left(1,\frac{1}{\sqrt{2\pi}}\int_{-\infty}^\infty|x|^q e^{-\frac{x^2}{2}}\,dx\right)^{\frac1q}=\begin{cases}
        1 & q\leq 2, \\
        \displaystyle\frac{\sqrt{2}}{\pi^{\frac{1}{2q}}}\Gamma\left(\frac{q+1}{2}\right), & q>2,
    \end{cases}
\end{equation}
with the Gamma function $\Gamma(x)=\int_0^\infty t^{x-1}e^{-t}\,dt$, $x>0$. 

The main results of this paper are stated in the following two theorems:

\begin{theorem}\label{thm:approximation_dbn}
    Let $q\in[1,\infty)$ and $f,\varphi\in\mathscr{D}_q(\R^d)$. Then, for any $\varepsilon>0$, there is an $M\in\N$ such that, for each $m\geq M$, we can find a $p\in\dbn_\varphi(d,m,m+1)$ satisfying
    \begin{equation*}
        \|f-p\|_{L^q}\leq\varepsilon.
    \end{equation*}
    If $q\in(1,\infty)$, then, for each $m\in\N$, the following quantitative bound holds:
    \begin{equation}\label{eq:lp_bound}
    \inf_{p\in\dbn_\varphi(d,m,m+1)}\|f-p\|_{L^q}\leq \frac{2 \Upsilon_q\|\varphi\|_{L^q}}{m^{1-\frac{1}{\min(q,2)}}},
    \end{equation}
    where the constant $\Upsilon_q$ is defined in \eqref{eq:cq}.
\end{theorem}
\begin{remark}\label{rem:gaussian}
    Returning to Example~\ref{ex:visible_density}, we find that $\|\varphi\|_{L^q}=q^{-\frac{d}{2q}}$ for the $d$-dimensional standard normal distribution \eqref{eq:gaussian} and 
    \begin{equation*}
        \|\varphi\|_{L^q}=\prod_{i=1}^d\frac{\lambda_i^{1-\frac{1}{q}}}{q^{\frac1q}\big(1-e^{-b_i\lambda_i}\big)}\left(1-e^{-q\lambda_ib_i}\right)^{\frac1q}
    \end{equation*}
    for the truncated exponential distribution \eqref{eq:exponential}. Our bound \eqref{eq:lp_bound} thus shows that deep belief networks with truncated exponential parental density (for suitable choice of the parameters $b$ and $\lambda$) better approximate the target density $f$. This is especially prevalent for small $q$, which is the primary case of interest, see Corollary~\ref{cor:krause} below.
\end{remark}

To state the approximation in the $L^\infty$-norm, we need to introduce the space of bounded and uniformly continuous functions:
\begin{equation*}
    \C_u(\R^d)=\left\{f\in L^\infty(\R^d):\,\lim_{\delta\downarrow 0}\sup_{|x-y|\leq\delta}\big|f(x)-f(y)\big|=0\right\}.
\end{equation*}
Notice that any probability density $f\in\mathscr{D}(\R^d)$, which is differentiable and has a bounded derivative, belongs to $\C_u(\R^d)$ since any uniformly continuous and  integrable function is bounded.
\begin{theorem}\label{thm:sup_norm}
    Let $f\in\mathscr{D}(\R^d)\cap\C_u(\R^d)$ and $\varphi\in\mathscr{D}_\infty(\R^d)$. Then, for any $\varepsilon>0$, there is an $M\in\N$ such that, for each $m\geq M$, we can find a $p\in\dbn_\varphi(d,m,m+1)$ satisfying
    \begin{equation*}
        \|f-p\|_{L^\infty}\leq\varepsilon.
    \end{equation*}
\end{theorem}

\begin{remark}
    The uniform continuity requirement on the target density in \cref{thm:sup_norm} can actually be relaxed to essential uniform continuity, that is, $f$ is uniformly continuous except on a set with zero Lebesgue measure. The most notable example of such a function is the uniform distribution $f=\1_{[0,1]}$.
\end{remark}

Another important metric between between probability densities $p,q:\R^d\to\R_+$ is the \emph{Kullback-Leibler divergence} (or \emph{relative entropy}) defined by
\begin{equation*}
    \mathrm{KL}(f\|g)=\int_{\R^d} f(x)\log\left(\frac{f(x)}{g(x)}\right)\,dx
\end{equation*}
if $\{x\in\R^d:\,g(x)=0\}\subset\{x\in\R^d:\,f(x)=0\}$ and $\mathrm{KL}(f\|g)=\infty$ otherwise. From \cref{thm:approximation_dbn,thm:sup_norm} we can deduce the following quantitative approximation bound in the Kullback-Leibler divergence:

\begin{corollary}\label{cor:krause}
   Let $\varphi\in\mathscr{D}_\infty(\R^d)$. Let $\Omega\subset\R^d$ be a compact set and $f:\Omega\to\R_+$ be a continuous probability density. Suppose that there is an $\eta>0$ such that both $f\geq\eta$ and $\varphi\geq\eta$ on $\Omega$. Then there is a constant $M>0$ such that, for each $m\in\N$, it holds that 
    \begin{equation}\label{eq:kl_bound}
        \inf_{p\in\dbn_\varphi(d,m,m+1)}\mathrm{KL}(f\|p)\leq\frac{M}{\eta m}\left(8\|\varphi\|_{L^2}^2+\|f-\varphi\|_{L^2(\Omega)}^2\right),
    \end{equation}
    where $\|f-\varphi\|_{L^2(\Omega)}^2=\int_\Omega \big|f(x)-\varphi(x)\big|^2\,dx$.
\end{corollary}

Let us note that any $\varphi\in\mathscr{D}_\infty(\R^d)$ is square-integrable so that the right-hand side of the bound \eqref{eq:kl_bound} is actually finite. This follows from the \emph{interpolation inequality}
\begin{equation}\label{eq:lp_interpolation}
    \|\varphi\|_{L^2}\leq \sqrt{\|\varphi\|_{L^1}\|\varphi\|_{L^\infty}}=\sqrt{\|\varphi\|_{L^\infty}},
\end{equation}
see \cite[Exercise 4.4]{Brezis2011}.

Corollary~\ref{cor:krause} considerably generalizes the results of \cite[Theorem 7]{Krause2013}: There, the authors only prove that deep belief networks can approximate any density in the closure of the convex hull of a set of probability densities satisfying certain conditions, which appear to be difficult to check in practice. That work also does not contain a convergence rate. In comparison, our results directly describe the class of admissible target densities and do not rely on the indirect description through the convex hull. Finally, there is an unjustified step in the argument of Krause et al., which appears hard to reconcile, see Remark~\ref{rem:krause} below for details.

\section{Proofs}

This section presents the proofs of Theorems~\ref{thm:approximation_dbn}, \ref{thm:sup_norm} and Corollary~\ref{cor:krause}. As a first step, we shall establish a couple of preliminary results in the next two subsections.

\subsection{$L^q$-Approximation of Finite Mixtures}

Given a set $A\subset L^q(\R^d)$, the \emph{convex hull} of $A$ is by definition the smallest convex set containing $A$; in symbols $\cv(A)$. It can be shown that 
\begin{equation*}
    \cv(A)=\left\{\sum_{i=1}^n\alpha_i a_i:\,\alpha=(\alpha_1,\dots,\alpha_n)\in\triangle_n,a_1,\dots,a_n\in A, n\in\N\right\}
\end{equation*}
with $\triangle_n=\big\{x\in[0,1]^n:\,\sum_{i=1}^nx_i=1\big\}$, the $n$-dimensional standard simplex.
It is also convenient to introduce the \emph{truncated convex hull}
\begin{equation*}
    \cv_m(A)=\left\{\sum_{i=1}^m\alpha_i a_i:\,\alpha=(\alpha_1,\dots,\alpha_m)\in\triangle_m,a_1,\dots,a_m\in A\right\}    
\end{equation*}
for $m\in\N$ so that $\cv(A)=\bigcup_{m\in\N}\cv_m(A)$. The \emph{closed convex hull} $\overline{\cv}(A)$ is the smallest closed convex set containing $A$ and it is straight-forward to check that it coincides with the closure of $\cv(A)$ in the topology of $L^q(\R^d)$.

The next result shows that we can approximate any probability density in the truncated convex hull of the set \eqref{eq:visible_density} arbitrarily well by a DBN with a fixed number of hidden units:

\begin{lemma}\label{lem:dbn_lp_bound}
  Let $q\in[1,\infty]$, $\varphi\in\mathscr{D}_q(\R^d)$, $\sigma>0$, and $m\in\N$. Then, for every $f\in\cv_m(\mathscr{V}_\varphi^{\sigma})$ and every $\varepsilon>0$, there is a deep belief network $p\in\dbn_\varphi(d,m,m+1)$ such that
  \begin{equation*}
    \|f-p\|_{L^q}\leq \varepsilon.
  \end{equation*}
\end{lemma}
\begin{proof}
    Since $f\in\cv_m(\mathscr{V}_\varphi^{\sigma})$, there are by definition of $\triangle_m$ $(\alpha_1,\dots,\alpha_m)\in\triangle_m$ and $(\mu_1,\dots,\mu_m)\in\big(\R^d\big)^m$ such that
    \begin{equation*}
        f=\sum_{i=1}^m \alpha_i\varphi_{\mu_i,\sigma}. 
    \end{equation*}
    We can think of $\alpha=(\alpha_1,\dots,\alpha_m)$ as a probability distribution $\tilde{\alpha}$ on $\{0,1\}^m$ by declaring
    \begin{equation*}
      \tilde{\alpha}(h_1)=\begin{cases}
      \alpha_i, & \text{if }h_1=e_i,\\
      0, & \text{else},
      \end{cases}
      \quad\qquad h_1\in\{0,1\}^m,
    \end{equation*}
    where $(e_i)_j=\delta_{i,j}$, $j=1,\dots,m$, is the $i^{\textup{th}}$ unit vector. 

    Let us fix $q\in[1,\infty]$ and $\sigma>0$. By Proposition~\ref{prop:approx_leroux}  there is a $\pi\in\brbm(m,m+1)$ such that
    \begin{equation}\label{eq:rbm_approximation}
        \left|\tilde\alpha(h_1)-\sum_{h_2\in\{0,1\}^{m+1}}\pi(h_1,h_2)\right|\leq\frac{\varepsilon}{m\sigma\|\varphi\|_{L^q}}\qquad\forall\,h_1\in\{0,1\}^m.
    \end{equation}
    We set 
    \begin{equation*}
        p(v\,|\,h_1)=\begin{cases}
            \varphi_{\mu_i,\sigma}(v), & h_1=e_i, \\
            0, & \text{else},
        \end{cases}
    \end{equation*}
    and
    \begin{equation*}
        p(v,h_1,h_2)=p(v\,|\,h_1)\pi(h_1,h_2)\in\dbn_\varphi(d,m,m+1).
    \end{equation*}
    This is the desired approximation since
    \begin{align*}
    \|f-p\|_{L^q}&\leq\sum_{i=1}^m\left|\alpha_i-\sum_{h_2\in\{0,1\}^{m+1}}\pi(e_i,h_2)\right|\big\|\varphi_{\mu_i,\sigma}\big\|_{L^q}\leq\varepsilon,
    \end{align*}
    where we used that $\|\varphi_{\mu,\sigma}\|_{L^q}=\sigma\|\varphi\|_{L^q}$ for each $\mu\in\R^d$ and each $\sigma>0$.
\end{proof}

\subsection{Approximation by Convolution}

Let $f\in L^q(\R^d)$, $q\in[1,\infty]$, and $\varphi\in\mathscr{D}(\R^d)$. We denote the \emph{convolution} of $f$ and $\varphi_\sigma(\cdot)=\sigma^{-d}\varphi\big(\sigma^{-1}\,\cdot\big)$ by
\begin{equation*}
    \big(f\star\varphi_\sigma\big)(x)=\int_{\R^d}f(\mu)\varphi_{\sigma}(x-\mu)\,d\mu=\int_{\R^d}f(\mu)\varphi_{\mu,\sigma}(x)\,d\mu.
\end{equation*}
Young's convolution inequality, \cite{Young1912}, implies $f\star\varphi_\sigma\in L^q(\R^d)$:
\begin{equation*}
    \big\|f\star\varphi_\sigma\big\|_{L^q}\leq\|f\|_{L^q}\|\varphi_\sigma\|_{L^1}=\|f\|_{L^q}.
\end{equation*}
In addition, the following approximation result holds:

\begin{proposition}\label{prop:conv_rate}
    Let $\varphi\in\mathscr{D}(\R^d)$. Then all of the following hold true:
    \begin{enumerate}
         \item\label{it:lp_approx} For each $q\in[1,\infty)$ and each $f\in L^q(\R^d)$, we have
        \begin{equation*}
            \lim_{\sigma\downarrow 0}\big\|f- f\star\varphi_\sigma\big\|_{L^q}=0.
        \end{equation*}
        \item\label{it:l_infty_approx} If $f\in L^\infty(\R^d)\cap\C_u(\R^d)$, then
        \begin{equation*}
            \lim_{\sigma\downarrow 0}\big\|f- f\star\varphi_\sigma\big\|_{L^\infty}=0.
        \end{equation*}
     \end{enumerate} 
\end{proposition}
\begin{proof}
    Item \ref{it:lp_approx} is well known, see e.g. \cite[Theorem 8.14]{Folland1999}. For \ref{it:l_infty_approx} fix $\varepsilon>0$. By uniform continuity of $f$, we can find a $\delta>0$ such that 
    \begin{equation*}
        \sup_{|\mu|\leq\delta}\big|f(x)-f(x-\mu)\big|\leq\frac{\varepsilon}{2}\qquad\forall\,x\in\R^d.
    \end{equation*}
    In particular, 
    \begin{equation*}
        \Big|f(x)-\big(f\star\varphi_\sigma\big)(x)\Big|\leq\int_{\R^d}\varphi_\sigma(\mu)\big|f(x)-f(x-\mu)\big|\,d\mu\leq 2\|f\|_{L^\infty}\int_{\{|\mu|>\delta\}}\varphi_{\sigma}(\mu)\,d\mu+\frac{\varepsilon}{2}.
    \end{equation*}
    Since 
    \begin{equation*}
        \int_{\{|\mu|>\delta\}}\varphi_{\sigma}(\mu)\,d\mu=\int_{\big\{|\mu|>\frac{\delta}{\sigma}\big\}}\varphi(\mu)\,d\mu\to 0\qquad\textup{as }\sigma\downarrow 0,
    \end{equation*}
    we can choose $\sigma_0>0$ such that $\big\|f-(f\star\varphi_\sigma)\big\|_{L^\infty}\leq\varepsilon$ for all $\sigma\in(0,\sigma_0)$. This completes the proof.
\end{proof}

\subsection{Approximation Theory in Banach Spaces}

The second ingredient needed in the proof of \cref{thm:approximation_dbn} is an abstract result from the geometric theory of Banach spaces. To formulate it, we need to introduce the following notion: The Rademacher type of a Banach space $\big(\cX,\|\cdot\|_{\cX}\big)$ the largest number $\t\geq 1$ for which there is a constant $C>0$ such that, for each $k\in\N$ and each $f_1,\dots, f_k\in\cX$,
\begin{equation*}
   \Expec{\bigg\|\sum_{i=1}^{k}\epsilon_i f_i\bigg\|^{\t}_{\cX}}\leq C\sum_{i=1}^k \|f_i\|_{\cX}^{\t}
\end{equation*}
holds, where $\epsilon_1,\dots,\epsilon_k$ are i.i.d. Rademacher random variables, that is, $\prob\big(\epsilon_1=\pm 1\big)=\frac12$. It can be shown that $\t\leq 2$ for every Banach space. 

\begin{example}\label{ex:rademacher_type}
    The space $L^q(\R^d)$ has Rademacher type $\t=\min(q, 2)$ for $q\in[1,\infty)$. The space $L^\infty(\R^d)$ on the other hand has only trivial type $\t=1$.
\end{example}
A good reference for the above results on the Rademacher type is \cite[Section 9.2]{Ledoux1991}. The next approximation result and its application to $L^q(\R^d)$ will become important in the sequel:

\begin{proposition}[{\cite{Donahue1997}, Theorem 2.5}]\label{prop:abstract_convex}
    Let $\big(\cX,\|\cdot\|_{\cX}\big)$ be a Banach space of Rade\-macher type $\t\in[1,2]$. Let $A\subset \cX$ and $f\in\overline{\cv}(A)$. Suppose that $\xi=\sup_{g\in A}\|f-g\|_{\cX}<\infty$.  Then there is a constant $C>0$ only depending on the Banach space $\big(\cX,\|\cdot\|_{\cX}\big)$ such that, for each $m\in\N$, we can find an element $g\in\cv_m(A)$ satisfying
    \begin{equation}\label{eq:convex_bound_abstract}
        \|f-g\|_{\cX}\leq\frac{C\xi}{m^{1-\frac{1}{\t}}}.
    \end{equation}
\end{proposition}
Notice that the bound \eqref{eq:convex_bound_abstract} is of course useless for $\t=1$. In addition, it can be shown that the convergence rate $m^{\frac{1}{\t}-1}$ is sharp, see \cite[Section 2.3]{Donahue1997}.

\begin{corollary}\label{cor:conv_lp_approx}
    Let $A\subset L^q(\R^d)$, $1\leq q<\infty$, and suppose that $f\in\overline{\cv}(A)$. If $\xi=\sup_{g\in A}\|f-g\|_{\cX}<\infty$, then for all $m\in\N$, there is a $g\in\cv_m(A)$ such that
    \begin{equation*}
        \|f-h\|_{L^q}\leq\frac{\Upsilon_q\xi}{m^{1-\frac{1}{\min(q,2)}}},
    \end{equation*}
    where $\Upsilon_q$ is the constant defined in \eqref{eq:cq}.
\end{corollary}
\begin{proof}
    Owing to \cref{ex:rademacher_type} we are in the regime of Proposition~\ref{prop:abstract_convex}. The sharp constant $C=\Upsilon_q$ was derived in \cite{Haagerup1981}.
\end{proof}

\subsection{Proof of \cref{thm:approximation_dbn,thm:sup_norm}}\label{sec:proof_main_result}

Before giving the technical details of the proofs, let us provide an overview of the strategy:
\begin{enumerate}
    \item By Proposition~\ref{prop:conv_rate} we can approximate the density $f\in\mathscr{D}_q(\R^d)$ with $f\star\varphi_\sigma$ up to an error which vanishes as $\sigma\downarrow 0$. 
    \item Upon showing that $f\star\varphi_\sigma\in\overline{\cv}(\mathscr{V}_\varphi^{\sigma})$, Proposition~\ref{cor:conv_lp_approx} allows us to show that for each $\varepsilon>0$ and each $m\in\N$, we can pick $\sigma>0$ such that
    \begin{equation*}
    \inf_{g\in\cv_m(\mathscr{V}_\varphi^{\sigma})}\|f-g\|_{L^q}\leq\varepsilon+\frac{2 \Upsilon_q\|\varphi\|_{L^q}}{m^{1-\frac{1}{\min(q,2)}}}.
    \end{equation*}
    \item Finally, we employ Lemma~\ref{lem:dbn_lp_bound} to conclude the desired estimate \eqref{eq:lp_bound}.
\end{enumerate}

\begin{lemma}\label{lem:convolution_closure}
    Let $q\in[1,\infty]$, $f\in\mathscr{D}_q(\R^d)$, and $\varphi\in\mathscr{D}(\R^d)$. Then, for each $\sigma>0$, we have
    \begin{equation*}
        f\star\varphi_\sigma\in\overline{\cv}(\mathscr{V}_\varphi^\sigma),
    \end{equation*}
    with the closure understood with respect to the norm $\|\cdot\|_{L^q}$.
\end{lemma}
\begin{proof}
    Let us abbreviate $g= f\star\varphi_\sigma$. We argue by contradiction. Suppose that $g\notin\overline{\cv}(\mathscr{V}_\varphi^\sigma)$. As a consequence of the Hahn-Banach theorem, $g$ is separated from $\overline{\cv}(\mathscr{V}_\varphi^\sigma)$ by a hyperplane. More precisely, there is a continuous linear function $\rho:L^q(\R^d)\to\R$ such that $\rho(h)<\rho(g)$ for all $h\in\overline{\cv}(\mathscr{V}_\varphi^\sigma)$, see \cite[Theorem. 1.7]{Brezis2011}. On the other hand, we however have
    \begin{equation*}
        \rho(g)=\rho\left(\int_{\R^d} f(\mu)\varphi_{\mu,\sigma}\,d\mu\right)=\int_{\R^d}f(\mu)\rho\big(\varphi_{\mu,\sigma}\big)\,d\mu<\rho(g)\int_{\R^d}f(\mu)\,d\mu=\rho(g),
    \end{equation*}
    which is the desired contradiction. 
\end{proof}

We can now establish the main results of this article:

\begin{proof}[\Cref{thm:approximation_dbn,thm:sup_norm}]
    Let us first assume that $q\in(1,\infty)$ and prove the quantitative bound \eqref{eq:lp_bound}. To this end fix $\varepsilon>0$ and $m\in\N$. We first observe that, by Proposition~\ref{prop:conv_rate}, we can choose $\sigma>0$ sufficiently small such that $\big\|f-f\star\varphi_\sigma\big\|_{L^q}\leq\frac{\varepsilon}{2}$. Employing Lemma~\ref{lem:convolution_closure} and Corollary~\ref{cor:conv_lp_approx} with $A=\mathscr{V}_\varphi^\sigma$, we can find a $g_m\in\cv_m(\mathscr{V}_\varphi^\sigma)$ such that
    \begin{equation*}
        \|f-g_m\|_{L^q}\leq\big\|f-f\star\varphi_\sigma\big\|_{L^q}+\big\|f\star\varphi_\sigma-g_m\big\|_{L^q}\leq \frac{\varepsilon}{2}+\frac{\Upsilon_q}{m^{1-\frac{1}{\min(q,2)}}}\sup_{\mu\in\R^d}\big\|f\star\varphi_\sigma-\varphi_{\mu,\sigma}\big\|_{L^q}.
    \end{equation*}
    For the last term we bound
    \begin{align*}
        \sup_{\mu\in\R^d}\big\|f\star\varphi_\sigma-\varphi_{\mu,\sigma}\big\|_{L^q}&=\sup_{\mu\in\R^d}\left(\int_{\R^d} \left|\int_{\R^d} f(x)\big(\varphi_\sigma(y-x)-\varphi_\sigma(y-\mu)\big)\,dx\right|^q\,dy\right)^{\frac1q} \\
        &\leq\int_{\R^d} f(x)\sup_{\mu\in\R^d}\left(\int_{\R^d}\big|\varphi_{\sigma}(y-x)-\varphi_{\sigma}(y-\mu)\big|^q\,dy\right)^{\frac1q} \\
        &=\sup_{\mu\in\R^d}\big\|\varphi-\varphi_{\mu,1}\big\|_{L^q}\leq 2\|\varphi\|_{L^q},
    \end{align*}
    whence
    \begin{equation*}
        \|f-g_m\|_{L^q}\leq\frac{\varepsilon}{2}+\frac{2\Upsilon_p\|\varphi\|_{L^q}}{m^{1-\frac{1}{\min(q,2)}}}.
    \end{equation*}
    Finally, Lemma~\ref{lem:dbn_lp_bound} allows us to choose $p\in\dbn_\varphi(d,m,m+1)$ such that $\|g_m-p\|_{L^q}\leq\frac{\varepsilon}{2}$. Therefore, we conclude
    \begin{equation*}
        \|f-p\|_{L^q}\leq\varepsilon+\frac{2\Upsilon_q\|\varphi\|_{L^q}}{m^{1-\frac{1}{\min(q,2)}}}.
    \end{equation*}
    Since $\varepsilon>0$ was arbitrary, the bound \eqref{eq:lp_bound} follows.

    If $q=1$ or $q=\infty$, we use the fact that 
    \begin{equation*}
        \overline{\cv}(A)=\overline{\bigcup_{m\in\N}\cv_m(A)}
    \end{equation*}
    for any subset $A$ of either $L^1(\R^d)$ or $L^\infty(\R^d)$, respectively. This implies that, for each $\varepsilon>0$, we can find $m\in\N$ and $g_m\in\cv_m(\mathscr{V}_\varphi^\sigma)$ such that $\big\|f\star\varphi_\sigma-g_m\big\|_{L^q}\leq\frac{\varepsilon}{3}$. If $q=\infty$, we note that a uniformly continuous and integrable function is always bounded. Hence, in any case we can apply Proposition~\ref{prop:conv_rate} to find a $\sigma>0$ for which $\big\|f-f\star\varphi_\sigma\big\|_{L^q}\leq\frac{\varepsilon}{3}$. Finally employing Lemma~\ref{lem:dbn_lp_bound} as above, there is a $p\in\dbn_\varphi(d,m,m+1)$ such that
    \begin{equation*}
        \big\|f-p\big\|_{L^q}\leq\big\|f-f\star\varphi_\sigma\big\|_{L^q}+\big\|f\star\varphi_\sigma-g_m\big\|_{L^q}+\big\|g_m-p\big\|_{L^q}\leq\varepsilon,
    \end{equation*}
    as required.
\end{proof}

\subsection{Kullback-Leibler Approximation on Compacts}\label{sec:kl}

Let us begin by bounding the Kullback-Leibler divergence in terms of the $L^2$-norm:
\begin{lemma}[\cite{Zeevi1997}, Lemma 3.3]\label{lem:bounds_kl}
    Let $\Omega\subset\R^d$, $f:\Omega\to\R_+$, and $g:\R^d\to\R_+$ be probability densities. If there is an $\eta>0$ such that both $f,g\geq\eta$ on $\Omega$, then 
    \begin{equation*}
        \mathrm{KL}(f\|g)\leq \frac{1}{\eta} \|f-g\|_{L^2(\Omega)}^2.
    \end{equation*}
\end{lemma}
\begin{proof}
    We use Jensen's inequality and the elementary fact $\log x\leq x-1$, $x>0$, to obtain
    \begin{align*}
        \mathrm{KL}(f\|g)&=\int_\Omega\log\left(\frac{f(x)}{g(x)}\right)f(x)\,dx\leq\log\left(\int_\Omega\frac{f(x)^2}{g(x)}\,dx\right) \\
        &\leq \int_\Omega\frac{f(x)^2}{g(x)}\,dx-1=\int_\Omega\frac{(f(x)-g(x))^2}{g(x)}\,dx\leq \frac{1}{\eta} \|f-g\|_{L^2}^2,
    \end{align*}
    as required.
\end{proof}

Finally, we can prove the approximation bound in Kullback-Leibler divergence:
\begin{proof}[Corollary~\ref{cor:krause}]
    Extending the target density $f$ by zero on $\R^d\setminus\Omega$, the corollary follows from \cref{thm:approximation_dbn} upon showing that, for each $m\in\N$, we can choose the approximation $p\in\dbn_\varphi(d,m,m+1)$ in such a way that $p\geq\frac{\eta}{2}$ on $\Omega$. 

    To see this, we notice that $f$ is uniformly continuous since $\Omega$ is compact. Hence, \cref{thm:sup_norm} allows us to pick an $M\in\N$ such that, for each $m\geq M$, there is a $p_m\in\dbn_\varphi(d,m,m+1)$ with $\|f-p_m\|_{L^\infty}\leq\frac{\eta}{2}$. In particular, each of these DBNs satisfies $p_m\geq\frac{\eta}{2}$ on $\Omega$. Consequently, by Lemma~\ref{lem:bounds_kl} we obtain
    \begin{equation}\label{eq:kl_combine1}
        \inf_{p\in\dbn_\varphi(d,m,m+1)}\mathrm{KL}(f\|p)\leq\frac{8\|\varphi\|_{L^2}^2}{\eta m}\qquad\forall m\geq M.
    \end{equation}

    A crude upper bound on $\inf_{p\in\dbn_\varphi(d,m,m+1)}\mathrm{KL}(f\|p)$ for $m<M$ can be obtained choosing both zero weights and biases in \eqref{eq:rbm_density} as well as $p(v\,|\,h_1)=\varphi$ for each $h_1\in\{0,1\}^m$ in \eqref{eq:dbn_equation}. Hence, the visible units of the DBN have density $\varphi$. This gives
    \begin{equation}\label{eq:kl_combine2}
        \inf_{p\in\dbn_\varphi(d,m,m+1)}\mathrm{KL}(f\|p)\leq\mathrm{KL}(f\|\varphi)\leq\frac{1}{\eta}\|f-\varphi\|_{L^2(\Omega)}^2\qquad\forall\,m=1,\dots,M-1,
    \end{equation}
    again by Lemma~\ref{lem:bounds_kl}. Finally, combining \eqref{eq:kl_combine1} and \eqref{eq:kl_combine2} we get
    \begin{equation*}
        \inf_{p\in\dbn_\varphi(d,m,m+1)}\mathrm{KL}(f\|p)\leq\frac{M}{\eta m}\left(8\|\varphi\|_{L^2}^2+\|f-\varphi\|_{L^2(\Omega)}^2\right),
    \end{equation*}
    as required.
\end{proof}

\begin{remark}\label{rem:krause}
    Our strategy of the proof of the Kullback-Leibler approximation in  Corollary~\ref{cor:krause} through Lemma~\ref{lem:bounds_kl} differs from the one employed in \cite[Theorem 7]{Krause2013}. There, the authors built on the results of \cite{Li1999} and in the course of their argument claim that the following statement holds true: \\[1em]
    {\normalfont Let $f_m,f:\Omega\to\R_+$, $m\in\N$, be probability densities on a compact set $\Omega\subset\R^d$ with $f_m,f\geq\eta>0$. If $\mathrm{KL}(f\|f_m)\to 0$ as $m\to\infty$, then $f_m\to f$ in the norm $\|\cdot\|_{L^\infty}$.} \\[1em]
    This however does not hold as we illustrate by the following simple counterexample: Let $\Omega=[0,1]$ and consider the sequence of probability densities given by
    \begin{equation*}
        f_m(x)=C_m\left(1\wedge \left(mx+\frac{1}{2}\right)\right),\qquad m\in\N,
    \end{equation*}
    where $C_m=(1-1/(8m))^{-1}$ is chosen such that $\int_0^1f_m(x)\,dx=1$. Then we have $f_m(x)\to\1_{[0,1]}(x)=f(x)$ pointwise on $(0,1]$ but certainly not uniformly. It is straight-forward to check that $\big\|f_m-\1_{[0,1]}\big\|_{L^2}\to 0$ and since $f_m,f\geq 1/2$ on $\Omega$, we have $\mathrm{KL}(f_m\|f)\to 0$ as $m\to\infty$ by Lemma~\ref{lem:bounds_kl}.
\end{remark}

\section{Conclusion}

We investigated the approximation capabilities of deep belief networks with two binary hidden layers of sizes $m$ and $m+1$, respectively, and real-valued visible units. We showed that, under minimal regularity requirements on the parental density $\varphi$ as well as the target density $f$, these networks are universal approximators in the strong $L^q$ and Kullback-Leibler distances as $m\to\infty$. Moreover, we gave sharp quantitative bounds on the approximation error. We emphasize that the convergence rate in the number of hidden units is independent of the choice of the parental density.

Our results apply to virtually all practically relevant examples thereby theoretically underpinning the tremendous empirical success of DBN architectures we have seen over the last couple of years. As we alluded to in Remark~\ref{rem:gaussian}, the frequently made choice of a Gaussian parental density does not provide the theoretically optimal DBN approximation of a given target density. Since, in practice, the choice of parental density cannot solely be determined from an approximation standpoint, but also the difficulty of the training of the resulting networks needs to be considered, it is interesting to further empirically study the choice of parental density on both artificial and real-world datasets.

\bibliographystyle{alpha}

\bibliography{lp_approx.bib}

\end{document}